\documentclass{article}

\usepackage[preprint]{neurips_2026}

\usepackage[utf8]{inputenc} % allow utf-8 input
\usepackage[T1]{fontenc}    % use 8-bit T1 fonts
\usepackage{hyperref}       % hyperlinks
\usepackage{silence}
\usepackage{url}            % simple URL typesetting
\usepackage{booktabs}       % professional-quality tables
\usepackage{amsfonts}       % blackboard math symbols
\usepackage{nicefrac}       % compact symbols for 1/2, etc.
\usepackage{microtype}      % microtypography
\usepackage{xcolor}         % colors
\usepackage{amsmath}
\usepackage{graphicx}
\usepackage{multirow}
\usepackage{amsmath, amssymb, amsthm}
\newtheorem{proposition}{Proposition}
\newtheorem{remark}{Remark}
\usepackage{tikz}
\newcommand*\circled[1]{\tikz[baseline=(char.base)]{
            \node[shape=circle,draw,inner sep=1pt] (char) {#1};}}
\usepackage{enumitem}
\setlist[itemize]{leftmargin=1.2em, itemsep=0pt, topsep=2pt}
\newcommand{\bottleneck}[1]{\textbf{\textsc{B#1}}}

\hypersetup{
  colorlinks=true,
  linkcolor=blue,
  citecolor=blue,
  urlcolor=blue
}

\title{Scaling Full Conformal Image Classifiers}

\author{%
  Julio~Silva-Rodr\'iguez \\
  Computer Vision Lab, ETH Zürich \\
  \texttt{jusilva@ethz.ch} \\
  \And
  Ender~Konukoglu \\
  Computer Vision Lab, ETH Zürich \\
}

\begin{document}

% TITLE, ABSTRACT

\maketitle          

% ---- Abstract ----
\begin{abstract}

Conformal prediction provides set-valued predictions with distribution-free coverage guarantees, making it attractive for high-stakes image classification. However, split conformal prediction is data-inefficient, while full conformal prediction (FCP), despite its stronger statistical efficiency, is computationally prohibitive at scale because it requires candidate-specific model refits at test time. We address this limitation by leveraging zero-shot vision-language models (VLMs) to guide scalable FCP in large label spaces. We introduce Targeted Full Conformal Prediction (T-FCP), which uses a lightweight inductive conformal predictor to prune unlikely labels and applies FCP only to the remaining candidates, reducing computation while retaining the formal guarantee of the combined conformal procedure. We further propose Stabilized Online LDA (SO-LDA), an efficient VLM adaptation solver based on rank-one inverse-covariance updates. Across multiple benchmarks, including ImageNet, T-FCP enables practical full-conformal image classification with modest test-time overhead, yielding efficient prediction sets and more stable empirical coverage than split conformal alternatives. Code is available: \href{https://github.com/jusiro/T-FCP}{T-FCP}
\end{abstract}

% ---- Introduction ----
\section{Introduction}
\label{sec:intro}

In image classification, an input image $\mathbf{x} \in \mathcal{X}$ is mapped to a label space $\mathcal{Y} = \{1, \dots, C\}$ via a trained model $\pi_{\theta}: \mathcal{X} \rightarrow \Delta_C$, where $\Delta_C$ denotes the probability simplex, and $\theta$ the parameters. Predictions are typically turned into decisions by selecting the label with the highest score. However, such a procedure does not provide information about model certainty, so potential errors or uncertain cases might remain unnoticed. While modern neural networks are increasingly adopted in high-stakes computer vision applications, e.g., medical imaging~\cite{litjens2017survey} or autonomous driving~\cite{liang2022effective}, there is a growing need to provide users with trustworthy operational decisions with appropriate guarantees. 

\textbf{Conformal prediction} (CP)~\cite{vovk_book,learning_ny_transduction,transductive_ci} is a machine learning framework that creates set-valued predictions with formal, distribution-free coverage guarantees. In the classification setting, CP constructs a prediction set $\mathcal{C}(\mathbf{x}) \subseteq \mathcal{Y}$ through a set-valued mapping $\mathcal{C}: \mathcal{X} \rightarrow 2^{\mathcal{Y}}$, such that the following marginal coverage property holds:
\begin{align}
\label{eq:marginal}
    \mathbb{P}\big(Y \in \mathcal{C}(\mathbf{x})\big) \geq 1 - \alpha,
\end{align}
where $\alpha \in (0,1)$ is a user-specified error rate. This guarantee ensures, under mild exchangeability assumptions, that the true label lies within the prediction set with probability at least $1-\alpha$ under the target data distribution. Importantly, the size of $\mathcal{C}(\mathbf{x})$ adapts to the uncertainty of the prediction, such that inputs more ambiguous according to the model typically yield larger sets

The construction of $\mathcal{C}$ is data-driven, as it relies on a \textit{labeled calibration dataset}, $\mathcal{D}_N = \{(\mathbf{x}_i, y_i)\}_{i=1}^N$. Specifically, CP requires defining a nonconformity score function $\mathcal{S}: (\mathcal{X} \times \mathcal{Y}) \times \theta_{\mathcal{D}_{*}} \rightarrow \mathbb{R}$, where $\mathcal{S}((\mathbf{x}, y); \theta_{\mathcal{D}_{*}})$ quantifies how atypical a candidate label $y$ is for a given input $\mathbf{x}$ relative to a model trained on a reference dataset $\mathcal{D}_{*}$. CP uses the calibration data to determine a rejection threshold via an empirical quantile search on the observed score distribution $s_i = \mathcal{S}((\mathbf{x}_i, y_i); \theta_{\mathcal{D}_{*}})$ for $i = 1, \dots, N$:
\begin{align}
\label{eq:threshold}
    \hat{q}_{\alpha} = \text{Quantile}({\mathcal{D}_N}; \mathcal{S}; \alpha)
    = \inf \left\{
    s \in \mathbb{R} :
    \frac{1}{N} \sum_{i=1}^N \mathbf{1}\{ s_i \leq s \}
    \geq 
    \frac{\lceil (N+1)(1-\alpha) \rceil}{N}
    \right\},
\end{align}
which corresponds to a finite-sample corrected $(1-\alpha)$-quantile of the calibration nonconformity scores, and is used for a new test input, $\mathbf{x}_{N+1}$, for constructing the prediction set:
\begin{align}
    \mathcal{C}(\mathbf{x}_{N+1}) = \{ y \in \mathcal{Y} : \mathcal{S}((\mathbf{x}_{N+1}, y); \theta_{\mathcal{D}_{*}}) \leq \hat{q}_{\alpha} \}.
\end{align}
By definition, this set contains all labels whose nonconformity scores are sufficiently small, and therefore \textit{conform} to the calibration distribution.

\textbf{Designing valid scoring functions.}
Several choices of $\mathcal{S}$ exist in the literature, which in turn influence the structure of the resulting prediction sets, typically building upon predicted class probabilities. As an illustrative example, one may define
$\mathcal{S}(\mathbf{x}, y; \theta) = 1 - \pi_\theta(\mathbf{x})_y,$
which assigns small scores to labels with high predicted probability and larger scores to less likely alternatives. A key formal requirement for conformal validity, however, is that the resulting nonconformity scores are \textit{exchangeable} when evaluated on the calibration and test samples. Formally, consider the augmented dataset, $\mathcal{D}_{N+1} = \{(\mathbf{x}_i, y_i)\}_{i=1}^{N+1}$, assumed to be exchangeable. Let $s_i = \mathcal{S}(\mathbf{x}_i, y_i; \theta)$ denote the corresponding scores. Then, for any permutation $\sigma$ of $\{1, \dots, N+1\}$, it holds that
\begin{align}
    \label{score_exchangeability}
    (s_1, \dots, s_{N+1}) \overset{d}{=} (s_{\sigma(1)}, \dots, s_{\sigma(N+1)}),
\end{align}
i.e., the joint distribution of the scores is \textit{permutation-invariant}. This condition ensures that the test score is statistically indistinguishable from calibration scores, yielding the marginal coverage guarantee in Eq.~\eqref{eq:marginal}. In practice, \textit{Inductive Conformal Prediction} (ICP)~\cite{inductive_ci,vovk_book} enforces this requirement by learning the model parameters $\theta$ on data that is disjoint from the calibration set, thereby ensuring that the scores $\{s_i\}_{i=1}^{N+1}$ are computed using a fixed, calibration data-independent model. When the calibration set is obtained by splitting the available labeled data, this procedure is commonly referred to as \textit{Split Conformal Prediction} (SCP)~\cite{inductive_ci,vovk_book,shafer2008tutorial,pmlr-v25-vovk12}.

\textbf{The data-efficiency challenge.} Despite the SCP being the most common procedure for image classification, it is also inherently data-inefficient, as only a fraction of the available data contributes to training and to the empirical quantile estimation. This limitation is nontrivial: the marginal coverage guarantee holds only in expectation over the calibration samples, introducing finite-sample variability (i.e., a beta-binomial distribution on $\alpha$ and $N$~\cite{pmlr-v25-vovk12,hulsman2022distribution,marques2025universal}) that may lead to noticeable deviations from the target coverage when $N$ is small. For instance, \textit{a user specifying $\alpha = 0.1$ may observe empirical coverage closer to $85\%$ instead of $90\%$ in practice.}

\textit{Full Conformal Prediction} (FCP)~\cite{transductive_ci} is, in contrast, a more appealing alternative in terms of data efficiency, as it leverages the entire available calibration labeled data for both model fitting and quantile estimation without splits. The key idea is to maintain the score exchangeability requirement in Eq. \ref{score_exchangeability} by ensuring that the fitting function is symmetric with respect to the extended dataset, $\mathcal{D}_{N+1}$, which means that for all $(\mathbf{x}_i, y_i) \in \mathcal{D}_{N+1}$, and any permutation $\sigma$:
\begin{align}
    \label{symmetry}
    \mathcal{S}((\mathbf{x}_i, y_i); \theta_{\mathcal{D}_{N+1}}) = \mathcal{S}((\mathbf{x}_i, y_i); \theta_{\sigma(\mathcal{D}_{N+1})}).
\end{align}
Since $y_{N+1}$ is unknown, the FCP procedure requires testing the conformity of each new test sample $\mathbf{x}$ and candidate label $y \in \mathcal{Y}$. Each augmented dataset $\mathcal{D}_{N+1}^y = \{(\mathbf{x}_1, y_1), \dots, (\mathbf{x}_N, y_N), (\mathbf{x}_{N+1}, y)\}$ is used to fit the parameters of the scoring function, $\theta_{\mathcal{D}_{N+1}^y}$, which produces its respective nonconformity scores, $s^{y}_i = \mathcal{S}((\mathbf{x}_i, y_i); \theta_{\mathcal{D}_{N+1}^y})$. The (1-$\alpha$)-quantile using such score function is obtained from the same calibration set, $\hat{q}_{\alpha}^y$, and is used as a rejection criterion to construct valid conformal sets:
\begin{align}
    \mathcal{C}_{\text{FCP}}(\mathbf{x}_{N+1}) = \{ y \in \mathcal{Y} : \mathcal{S}((\mathbf{x}_{N+1}, y); \theta_{\mathcal{D}_{N+1}^y}) \leq \hat{q}_{\alpha}^y \}.
\end{align}
Despite being more data-efficient, the FCP procedure creates two computational bottlenecks at test time: (\bottleneck{1}) testing each potential label for an incoming test point, and (\bottleneck{2}) the cost of training each classifier during testing. Consequently, especially considering standard large-scale computer vision benchmarks, e.g., ImageNet, where $C=1000$, FCP has been systematically neglected in favor of SCP, at the cost of assuming larger annotation efforts or wider coverage distributions. 

\textbf{Contributions.} In this work, we aim to enable the construction of full conformal image classifiers at scale. Our main observation is that in a large label space, despite not being intrinsically ordered, many classes may be weakly related to each other, e.g., “husky” and “toaster”, while some groups of categories might be more similar, e.g., “husky”, “malamute”, “Samoyed”, or “Greater Swiss Mountain dog” (\textit{real examples from ImageNet}). Therefore, if an initial proxy were available to guide the FCP procedure toward the most likely labels for a new test image, the total runtime could be significantly reduced. Hence, we propose exploiting the open-vocabulary capabilities of pre-trained vision-language models (VLMs), such as CLIP~\cite{radford2021learning}. These provide zero-shot class prototypes via text descriptions for the target label domain, which, despite not being specialized in the downstream task, serve as proxies in the model's embedding space to discard non-similar categories that are not considered in the FCP procedure for each test image. Technically, our contributions are:
\vspace{-1mm}
\begin{itemize}
    \item We formalize \textit{Targeted Full Conformal Prediction} (T-FCP), a procedure that only requires running FCP on a subset of the label space (\bottleneck{1}), defined by an ICP procedure guided by zero-shot class prototypes with provable coverage guarantees without requiring splits on calibration data.
    \item To address the cost of multiple model fits (\bottleneck{2}), we introduce SO-LDA, an online Linear Discriminant Analysis solver for VLM adaptation that performs stable incremental class-prototype updates together with rank-one inverse-covariance updates. SO-LDA achieves performance close to expensive gradient-based solvers while being orders of magnitude faster.
\end{itemize}
\vspace{-1mm}
Extensive experiments demonstrate that our framework enables full conformal procedures on large-scale datasets with a modest latency overhead (e.g., $\sim30 \,\text{ms}$/image for ImageNet), yielding efficient sets and more stable coverage distributions than split conformal procedures.

% ---- Related work ----
\section{Related Work}
\label{sec:rw}

\textbf{Full conformal prediction.} Recent work on full conformal prediction has focused on reducing its computational burden, e.g., by limiting the candidate label space evaluated at test time. In regression, this has been achieved by exploiting the natural ordering of the response space, either through discretization~\cite{lei2013distribution,Lei2018} or homotopy/continuation methods that track the solution path of the fitted model as the candidate response varies~\cite{Lei2018,lei2019fast,guha2023conformalization,li2024generalized}. Classification problems, however, do not directly admit these strategies. Existing work for classification has instead focused on online classifiers with decoupled test-time updates, including exact permutation-invariant updates for kNN and SVMs~\cite{cherubin2021exact}, as well as inexact updates for gradient-based solvers via influence functions~\cite{martinez2023approximating}.

\textbf{Conformal prediction for image classification.}
Most neural-network-based image classification works follow SCP~\cite{einbinder2022training,conftr,raps,ding2024class,saps,pmlr-v258-campos25a,ding2026conformal}, improving either training objectives for set efficiency~\cite{einbinder2022training,conftr} or nonconformity scores for adaptiveness, class-conditional coverage, and robustness~\cite{raps,pmlr-v258-campos25a,ding2024class,ding2026conformal}. In contrast, calibration-data efficiency and empirical coverage variability have received less attention, as these methods typically rely on in-domain training data and separate calibration splits.

\textbf{Conformal prediction for zero-shot models.} Recent works have explored split conformal prediction for vision foundation models~\cite{fillioux2024foundation} and, more specifically, vision-language models (VLMs)~\cite{confot25}. Several methods improve VLM adaptation within conformal procedures using unsupervised transductive adaptation~\cite{confot25,scat,lata}. Closer to our work,~\cite{fca25} explored full conformal classifiers using closed-form linear solvers in the VLM embedding space. However,~\cite{fca25} focused on small-scale datasets ($C<50$). As we show later, their proposed kNN-style solver, SS-Text, does not scale well to large label spaces and fine-grained classification tasks.

\textbf{Combining conformal predictors.} Related to our work are methods that combine or cascade multiple split conformal procedures~\cite{fisch2021efficient,timans2024adaptive,SanSar_Conformal_MICCAI2025}. For example, conformal cascades progressively prune candidate sets obtained with different nonconformity scores~\cite{fisch2021efficient}. Other works derive combined coverage guarantees, either through multiplicative bounds under independence assumptions~\cite{timans2024adaptive}, or by allocating individual error rates using common corrections such as Bonferroni~\cite{timans2024adaptive,SanSar_Conformal_MICCAI2025}. These approaches have been used to achieve goals such as label-dependent guarantees for object detection~\cite{timans2024adaptive} and joint guarantees across multiple outputs~\cite{SanSar_Conformal_MICCAI2025}.

% ---- Methods ----
\section{Methods}
\label{sec:methods}

\subsection{Background on zero-shot image classifiers}

Let $f_{\theta}: \mathcal{X} \rightarrow \mathbb{R}^F$ denote a visual encoder that maps an input image $\mathbf{x}$ to a feature vector $\mathbf{v} = f_{\theta}(\mathbf{x}) \in \mathbb{R}^F$. We consider contrastive vision-language models (VLMs), following CLIP~\cite{radford2021learning}, as our \emph{zero-shot models}. These provide class prototypes from class descriptions (e.g., ``an image of a \texttt{[class name]}''), via its text encoder $g_{\phi}$, producing textual embeddings $\mathbf{t}_y = g_{\phi}(\text{prompt}_y) \in \mathbb{R}^F$. These embeddings can be directly used as class prototypes, i.e., $\mathbf{W}^0 = (\mathbf{t}_y)_{y=1}^C$, by measuring the similarity between $\ell_2$-normalized image embeddings and class prototypes, yielding class probabilities:
\begin{align}
\label{eq:probs}
    \pi_{\mathbf{W}^0}(\mathbf{x}_i) = (p^0_{i,y})_{y=1}^C \in \Delta_C, 
    \quad
    p^0_{i,y} = \frac{\exp\big(\mathbf{v}_i^\top \mathbf{w}^0_y/\tau \big)}
    {\sum_{j \in \mathcal{Y}} \exp\big(\mathbf{v}_i^\top \mathbf{w}^0_j/\tau \big)},
\end{align}

\subsection{Targeted Full Conformal Prediction}
\label{sec:tfcp}
First, we aim to leverage zero-shot models to address the first bottleneck in full conformal prediction procedures (\bottleneck{1}): the need to iterate over the entire label space for each incoming test sample.

\textbf{Combining conformal procedures.}
Our approach is based on the observation that multiple valid conformal predictors can be combined while maintaining statistical guarantees. Let $\mathcal{C}_1$ and $\mathcal{C}_2$ be two conformal predictors constructed using error levels $\alpha_1$ and $\alpha_2$, respectively. That is, for a test point $(\mathbf{x}, Y)$, and under the standard exchangeability assumptions:
\begin{align}
\label{alpha1alpha2}
    \mathbb{P}(Y \in \mathcal{C}_1(\mathbf{x})) \geq 1 - \alpha_1,
    \quad
    \mathbb{P}(Y \in \mathcal{C}_2(\mathbf{x})) \geq 1 - \alpha_2.
\end{align}
We define a combined prediction set as the intersection
\begin{align}
    \label{combination}
    \mathcal{C}_{1\cap2}(\mathbf{x}) = \mathcal{C}_1(\mathbf{x}) \cap \mathcal{C}_2(\mathbf{x}).
\end{align}
\begin{proposition}
\label{prop:intersection}
The prediction set $\mathcal{C}_{1\cap2}(\mathbf{x})$ satisfies
\begin{align}
    \mathbb{P}\big(Y \in \mathcal{C}_{1\cap2}(\mathbf{x})\big) \geq 1 - (\alpha_1 + \alpha_2).
\end{align}
\end{proposition}
The proof follows directly from De Morgan's law and the union bound, and is given in Appendix~\ref{app_proof}.
\begin{remark}[Tightness and conservativeness]
\label{remark:tightness}
The bound in Proposition~\ref{prop:intersection} is a worst-case guarantee and may be conservative in practice. It becomes tight when the two miscoverage events are disjoint and have probabilities \(\alpha_1\) and \(\alpha_2\). When the error events overlap, the intersection can have coverage higher than \(1-(\alpha_1+\alpha_2)\). If the individual miscoverage probabilities are exactly \(\alpha_1\) and \(\alpha_2\), the coverage of the intersection lies between \(1-(\alpha_1+\alpha_2)\) and \(1-\max\{\alpha_1,\alpha_2\}\).
\end{remark}

\textbf{Targeted conformal prediction via zero-shot models.} Let us now introduce a labeled calibration set consisting of embedding representations, $\mathcal{D}_N = \{(\mathbf{v}_i, y_i)\}_{i=1}^N$. A set of zero-shot linear weights, $\mathbf{W}^0$, is presented, and the goal is to create conformal sets with coverage guarantees from new incoming test images, using their embedding, $\mathbf{v}_{N+1}$. Let us now define an ICP procedure, $\mathcal{C}_{\text{ICP}}$, whose nonconformity score function leverages the zero-shot weights, $\mathcal{S}(\mathbf{x}, y; \mathbf{W}^{0})$, and produces conformal sets with an error rate of $\alpha_{\text{ICP}}$ using the whole calibration set. Also, let us denote an FCP procedure, $\mathcal{C}_{\text{FCP}}$, that uses a more expensive scoring function by fitting the classifier on the extended dataset $\mathcal{S}((\mathbf{x}, y); \mathbf{W}^{\mathcal{D}_{N+1}^y})$, operating at an error level $\alpha_{\text{FCP}}$. We define the \textit{Targeted Full Conformal Prediction} (T-FCP) sets as the intersection of both procedures, i.e., $\mathcal{C}_{\text{T-FCP}}=\mathcal{C}_{\text{ICP}\cap\text{FCP}}$, such that, according to Proposition~\ref{prop:intersection}, the following marginal coverage property holds:
\begin{align}
    \mathbb{P}\big(Y \in \mathcal{C}_{\text{T-FCP}}(\mathbf{x})\big) \geq 1 - (\alpha_{ICP} + \alpha_{FCP}).
\end{align}
By construction, any label excluded by $\mathcal{C}_{\text{ICP}}$ for a test sample cannot appear in $\mathcal{C}_{\text{T-FCP}}$. Therefore, we can use $\mathcal{C}_{\text{ICP}}$ as a \textit{pruning stage} of highly unlikely labels to be included in the conformal set, and only run the expensive FCP procedure, $\mathcal{C}_{\text{FCP}}$, on the surviving, i.e., \textit{target labels}, while maintaining theoretical guarantees. Specifically, for a given target error rate $\alpha$, we can set $\alpha_{FCP} = \alpha - \alpha_{ICP}$. The parameter $\alpha_{\text{ICP}}$ controls a computational-statistical tradeoff: larger values typically yield smaller ICP candidate sets and therefore lower FCP cost, but allocate less error budget to the full conformal stage. In our experiments, we find that \(\alpha_{\mathrm{ICP}}=0.5\%\) already provides substantial computational efficiency gains while keeping the additional conservativeness limited in practice.

\subsection{SO-LDA: Stabilized Online LDA}
\label{sec:lda}

Second, we exploit the capabilities of zero-shot models to alleviate the second bottleneck in full conformal prediction (\bottleneck{2}): the computational cost of fitting the classifier for each extended dataset. 

\textbf{Transfer learning.} Instead of training a deep network for each extended dataset, we leverage the rich embedding representation of pre-trained vision-language models, similarly to~\cite{fca25}. Given a dataset with labeled visual examples, and its corresponding embeddings, $\mathcal{D}_N = \{(\mathbf{v}_i, y_i)\}_{i=1}^N$, a common strategy is to learn task-specific linear class prototypes $\mathbf{W}^{\mathcal{D}_N}$, that replace its zero-shot counterpart in Eq.~\ref{eq:probs} to produce softmax class scores. 

\textbf{Online LDA.} While several linear probe solvers have been proposed in the literature~\cite{gao2021clip,clap24,wang2024a}, gradient-descent-based minimization of softmax cross-entropy is generally the best-performing solution~\cite{lin2023multimodality,clap24}. However, its computational cost ($\mathcal{O}(T NFC)$, with $T$ the number of iterations) makes it infeasible for full conformal prediction. We therefore turn our attention to training-free solvers~\cite{zhang2021tip,wang2024a,fca25,sstext25}, particularly those that enable decoupled updates between calibration samples and new incoming test points. Specifically, we build upon the Linear Discriminant Analysis (LDA) solver~\cite{lda}, proposed in~\cite{wang2024a} for VLMs, which models class-conditional feature distributions as Gaussian with a shared covariance matrix. In its simplified form, by dropping the bias term for consistency with Eq.\ref{eq:probs}, the classifier weights are given by:
\begin{align}
\label{lda_classweight}
\mathbf{w}_y = \mathbf{\Sigma}^{-1} \mathbf{\mu}_y
\end{align}
where $\mu_{y}$ are the class centers, and $\Sigma^{-1}$ the inverse covariance matrix. 

Given a labeled set $\mathcal{D}_{N}$, and let $N_y=\sum_{i\in\mathcal{D}_N}\mathbf{1}\{ y_i = y \}$ denote the number of samples for each category, then the model parameters can be estimated as follows:
\begin{align}
{\mathbf{\mu}_y}_{(N)} = \frac{1}{N_{y}}\sum_{i\in\mathcal{D}_N} \mathbf{1}\{ y_i = y \}\mathbf{v}_i, \quad 
\mathbf{\Sigma}^{-1}_{(N)} = \mathbf{S}_{(N)}^{-1} = (\frac{1}{N}\sum_{i\in\mathcal{D}_N} \mathbf{z}_i \mathbf{z}_i^{\top} )^{-1},
\end{align}
where $\mathbf{z}_i=\mathbf{v}_i-{\mathbf{\mu}_{y=y_i}}_{(N)}$ are the class-centered features. To avoid numerical instabilities, we use diagonal-loading stabilization, $\mathbf{S}^{\text{reg}}_{(N)}=\mathbf{S}_{(N)}
+\lambda_{\mathrm{REG}}\operatorname{Diag}(\mathbf{S}_{(N)})$, following the common use of shrinkage for high-dimensional precision estimation~\cite{kubokawa2008estimation,wang2024a}. The loading term is estimated once from the calibration covariance and kept fixed during candidate augmentations, preserving efficient rank-one updates. Appendix~\ref{app_solda} compares it with fixed-ridge and full recomputation variants.

Given a new incoming sample, ($\mathbf{v}_{N+1}, y_{N+1}$), the model parameters can be updated~\cite{kuncheva2008adaptive}, instead of being recalculated again. Particularly, for the covariance matrix updates, we can consider the Sherman–Morrison formula~\cite{hager1989updating} to compute the inverse of its rank-one update efficiently:
\begin{align}
{\mathbf{\mu}_y}_{(N+1)}
&=
\frac{
N_{y}\cdot{\mathbf{\mu}_y}_{(N)}
+
\mathbf{v}_{N+1}\mathbf{1}\{ y_{N+1} = y \}
}{
N_{y} + \mathbf{1}\{ y_{N+1} = y \}
},
\\
\mathbf{\Sigma}^{-1}_{(N+1)}
&=
\frac{N+1}{N}
\left(
\mathbf{\Sigma}^{-1}_{(N)}
-
\frac{
\mathbf{\Sigma}^{-1}_{(N)}
\mathbf{z}_{N+1}\mathbf{z}_{N+1}^{\top}
\mathbf{\Sigma}^{-1}_{(N)}
}{
N+\mathbf{z}_{N+1}^{\top}
\mathbf{\Sigma}^{-1}_{(N)}
\mathbf{z}_{N+1}
}
\right).
\label{lda_updates}
\end{align}
In full conformal procedures, the parameters are first estimated from the calibration data and then updated for each augmented dataset $\mathcal{D}_{N+1}^y$. This reduces the per-label update cost to $\mathcal{O}(F^2)$, avoiding repeated matrix inversion ($\mathcal{O}(F^3)$) and enabling scalable conformal inference compared to recomputing the LDA solution from scratch ($\mathcal{O}(NF^2 + F^3)$).

\textbf{Enabling stable online updates through zero-shot prototypes.} A practical difficulty of online LDA in the full conformal setting is that the residual vectors $\mathbf{z}_i$ for samples in $\mathcal{D}_N$ are computed using statistics estimated before the candidate test sample is observed, whereas the test sample itself is evaluated under updated statistics. This creates an asymmetry between the approximate online update (Eq.~\ref{lda_updates}) and a full refit on the augmented dataset. To mitigate this issue, we introduce \textbf{S}tabilized \textbf{O}nline \textbf{LDA} (\textbf{SO-LDA}), including the following adjustments:\\
\begin{minipage}{0.48\linewidth}
\begin{align}
\mathbf{z}_i 
= 
\mathbf{v}_i - \mathbf{W}_{y_i}^{(0)}
\label{eq:so_residual},
\end{align}
\end{minipage}
\hfill
\begin{minipage}{0.48\linewidth}
\begin{align}
\mathbf{\mu}_y^{\text{SO-LDA}}
=
\mathbf{\mu}_y^{(N+1)}
+
\lambda_{\text{TEXT}} \mathbf{W}_y^{(0)}
\label{eq:so_mean}.
\end{align}
\end{minipage}

According to Eq.~\ref{eq:so_residual}, the feature-centering anchors used for covariance estimation no longer depend on the calibration samples, but on the zero-shot prototypes. Therefore, the residual centering rule remains invariant under candidate augmentation and avoids the asymmetry induced by recomputing visual class centers differently for calibration and test samples. Additionally, following standard practices in VLM adaptation~\cite{clap24,sstext25,fca25}, Eq.~\ref{eq:so_mean} introduces a text-guided prior (weighted by $\lambda_{\text{TEXT}}$) on the updated class means by biasing them toward the textual prototypes before $\ell_2$-normalization and score prediction in Eq.~\ref{lda_classweight}. This improves stability, especially in low-data regimes.

% ---- Experiments ----
\section{Experiments}
\label{sec:experiments}

\subsection{Setup}
\label{ssec:setup}

\textbf{Datasets.} We evaluate the proposed procedure in the standard 11 datasets used for CLIP zero- and few-shot transfer evaluation~\cite{clap24,wang2024a}. These are: ImageNet~\cite{deng2009imagenet}, SUN397~\cite{sun397}, FGVCAircraft~\cite{aircraft}, EuroSAT~\cite{eurosat}, StanfordCars~\cite{stanfordcars}, Food101~\cite{food101}, OxfordPets~\cite{oxfordpets}, Flowers102~\cite{flowers102}, Caltech101~\cite{caltech}, DTD~\cite{dtd}, and UCF101~\cite{ucf101}. These gather a heterogeneous set of general, fine-grained, and specialized domain classification tasks, several of which have hundreds of categories. We refer to Appendix~\ref{app_datasets} for specific details on the number of categories and tasks. 

\textbf{Calibration sets.} The corresponding test partition from each dataset is used in our conformal experiments to produce disjoint calibration and test subsets. Specifically, samples of size $N=C\times K$, with $K$ the so-called number of shots in the transfer learning literature~\cite{zhou2022coop,gao2021clip,clap24}, are retrieved for calibration. By default, we set $K=16$, a standard data regime in transfer learning from VLMs~\cite{zhou2022coop,gao2021clip,clap24}, and present ablation studies with smaller sample sizes. Note that, in contrast to the standard literature on transfer learning from VLMs, which samples label-balanced adaptation sets, we follow the practice in conformal prediction~\cite{confot25,scat} that maintains the dataset's label-marginal distribution in the calibration data to avoid breaking exchangeability assumptions. All experiments are repeated 50 times using different random seeds for sampling the calibration data.

\textbf{Zero-shot models.} CLIP \cite{radford2021learning}, specifically ViT-B/16 backbone, is used as a zero-shot model to evaluate the proposed conformal procedure. We also present generalization studies to other CLIP and MetaCLIP~\cite{xu2024demystifying} backbones. The text encoder from each model is used to produce zero-shot class-wise prototypes for each downstream category by using standard templates and category names \cite{zhou2022coop,gao2021clip,clap24}, specifically the same prompts as in~\cite{confot25}. 

\textbf{Implementation details.} Conformal sets are created at error rates $\alpha\in\{0.10, 0.05\}$. For T-FCP, we set $\alpha_{\text{ICP}}=0.5\%$. For SO-LDA, the hyperparameters are fixed to $\lambda_{\text{TEXT}}=1$ and $\lambda_{\text{REG}}=10$.

\textbf{Baselines.} First, we compare our conformal procedure with standard ones in the literature:
\begin{itemize}
    \item Inductive conformal prediction (ICP): uses the whole calibration sample for quantile search, and computes the nonconformity score using the zero-shot scores (ZS) without adaptation.
    \item Split conformal prediction (SCP): splits the calibration data into two equally-sized subsets, one for adaptation and the other for conformal quantile search. Since SCP is inductive, there is no computational restriction during the adaptation stage; therefore, we used standard softmax cross-entropy minimization via iterative gradient descent (GD). Training details are in Appendix~\ref{app_baselines}.
    \item Full conformal prediction (FCP): performs the full conformal procedure over the whole label space. Adaptation is performed using our proposed SO-LDA solver.
\end{itemize}
Second, we include ICP procedures that perform unsupervised transductive adaptation (ICP-T), recently proposed for VLMs~\cite{confot25, scat,lata}. Training details are in Appendix~\ref{app_baselines}.

\textbf{Nonconformity scores.}
We use LAC~\cite{lac} as the default score, as it is low-latency and commonly used in FCP~\cite{cherubin2021exact,martinez2023approximating}. APS~\cite{aps} is explored in Appendix~\ref{results_aps}. More adaptive scores, such as RAPS~\cite{raps} or ClusterCP~\cite{ding2024class}, require sorting or additional hyperparameters, which introduces extra computation and, in some cases, additional data splits. Since our focus is on data-efficient and scalable FCP, we leave the integration of adaptive nonconformity scores into large-scale FCP to future work.

\textbf{Metrics.}
We report accuracy and standard conformal metrics. Empirical coverage is characterized by its average and deviation (\(2\sigma\)) across seeds, together with the percentage of runs achieving coverage at least \(1-\alpha-0.005\) (\%Valid), being $0.005$ a tolerance. Set efficiency is measured by the mean and median set size, and by the percentage of singleton predictions (\%Sing.), as in~\cite{pmlr-v258-campos25a}.

\subsection{Main results}
\label{ssec:results}

\textbf{Standard conformal procedures.} Table~\ref{tab:conformal_results} provides the results obtained with the different conformal procedures, and Figure~\ref{fig:cov_per_dataset}~(a) illustrates the empirical coverage distribution for the main conformal alternatives. \textbf{\circled{1}}~In terms of \textbf{coverage}, it is worth noting that, despite all procedures achieving the target average coverage, SCP's coverage deviation across samples is $20\%$ relatively larger than ICP's, due to the data split. This is not the case for FCP procedures, which provide narrower coverage distributions, as in ICP. Notably, T-FCP achieves the largest proportion of experiments with valid coverage among these procedures: $81.8\%$ and $88.7\%$ for $\alpha=0.10$ and $\alpha=0.05$, respectively. This is explained by the theoretically expected slightly over-coverage produced by the union-bound error rate design, i.e., $\alpha_{\text{ICP}}$, for early label pruning. \textbf{\circled{2}}~Regarding \textbf{set efficiency}, ICP produces the largest set sizes, since it does not benefit from transfer learning with labeled data. SCP and FCP provide similar mean set sizes. However, the median indicates that FCP produces set sizes that are $\sim 10\%$ smaller (more efficient) than SCP, while its number of singletons (better on the discriminative aspect) is larger by a considerable margin. We explain these mixed observations in Figure~\ref{tab:conformal_results}~(b), showing that FCP using SO-LDA presents longer tails in the set size distribution, i.e., larger sets on uncertain inputs. This is explained by the design of the FCP, which inflates the conformity of each image-label candidate by training in the extended dataset, and the differences in optimization objectives between solvers. Finally, T-FCP provides slightly more conservative (larger) sets than FCP, with the advantage of enabling its real-time application on large-scale datasets, as discussed later. Still, T-FCP improves the median set size and singleton rate relative to SCP.

\begin{table}[h!]
\setlength{\tabcolsep}{2.6pt}  % default is ~6pt
\centering
\caption{Comparison of conformal procedures atop CLIP ViT-B/16 using $N=C\times16$. Results are averaged across 11 datasets, and per-dataset results are in Appendix \ref{app_detailedresults}.}
\vspace{-1mm}
\label{tab:conformal_results}
\small
\begin{tabular}{lll | c | ccc ccc | ccc ccc}
\toprule
\multicolumn{4}{c}{}
& \multicolumn{3}{c}{Cov. ($\alpha = 0.10$)} 
& \multicolumn{3}{c}{Set size}
& \multicolumn{3}{c}{Cov. ($\alpha = 0.05$)} 
& \multicolumn{3}{c}{Set size} \\
\cmidrule(lr){5-7} \cmidrule(lr){8-10}
\cmidrule(lr){11-13} \cmidrule(lr){14-16}
& & & Acc. 
& Avg. & $2\sigma$ & \%Valid 
& Mean & Med. & \%Sing. 
& Avg. & $2\sigma$ & \%Valid 
& Mean & Med. & \%Sing. \\
\midrule
ICP             & ZS                  & & 65.7 & 90.1 & 2.0 & 77.8 & 4.8 & 4.5 & 37.3 & 95.2 & 1.5 & 84.0 & 7.9 & 7.3 & 28.2 \\
ICP-T           & OT~\cite{confot25}  & & 68.6 & 90.6 & 2.0 & 82.3 & 3.9 & 3.7 & 42.4 & 95.4 & 1.4 & 86.8 & 6.1 & 5.6 & 32.6 \\
ICP-T           & TIM~\cite{scat}     & & 71.5 & 90.6 & 2.0 & 80.5 & 3.6 & 3.5 & 45.2 & 95.4 & 1.4 & 87.2 & 5.6 & 5.3 & 34.3 \\
\midrule
SCP             & GD                  & & 79.9 & 90.7 & 2.7 & 77.1 & 2.2 & 2.1 & 64.0 & 95.5 & 1.9 & 79.5 & 3.3 & 3.0 & 51.9 \\
\midrule
\textbf{FCP}             & \textbf{SO-LDA}               & & 80.3 & 90.1 & 2.2 & 77.6 & 2.2 & 1.8 & 69.3 & 95.2 & 1.6 & 80.5 & 3.3 & 2.6 & 57.4 \\ 
\textbf{T-FCP}           & \textbf{SO-LDA}               & & 80.3 & 90.7 & 2.1 & 81.8 & 2.3 & 1.9 & 68.1 & 95.6 & 1.5 & 88.7 & 3.5 & 2.7 & 56.0 \\
\bottomrule
\end{tabular}
\end{table}

\begin{figure*}[h!]
    \begin{center}
        \setlength{\tabcolsep}{1.4pt}
        \begin{tabular}{cccc}
        
         \includegraphics[width=.24\linewidth]{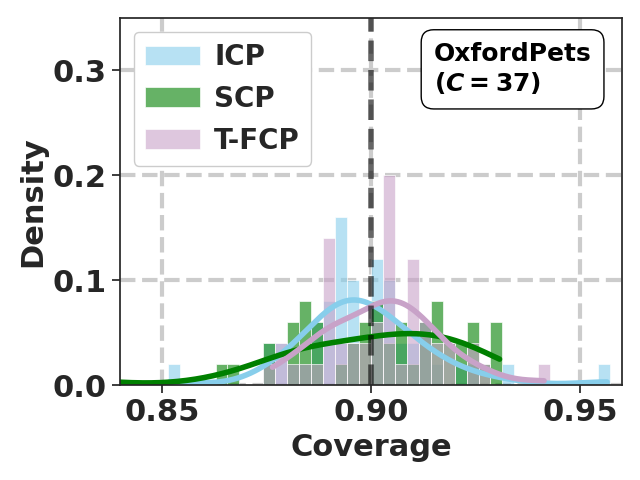} &
         \includegraphics[width=.24\linewidth]{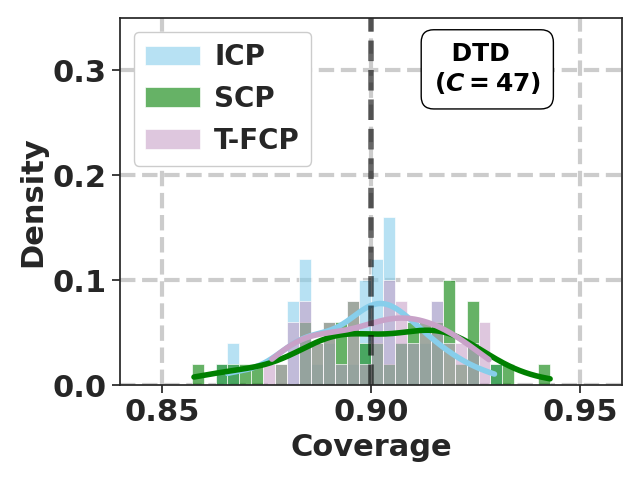} &
         \includegraphics[width=.24\linewidth]{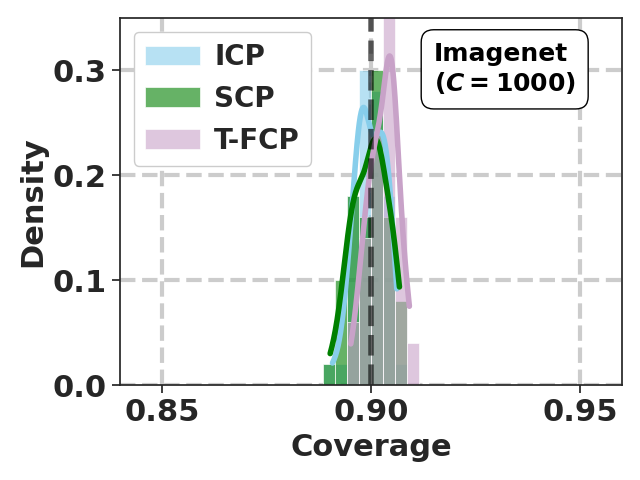} &
        \includegraphics[width=0.24\linewidth]{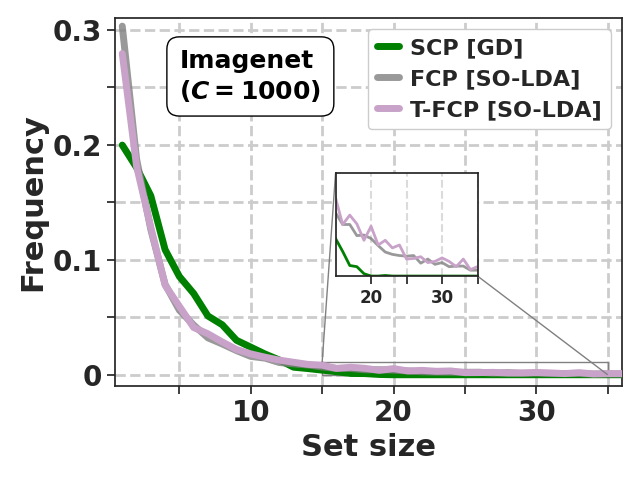}  \\ 

        \multicolumn{3}{c}{\small{(a) Coverage empirical distribution.}}  &
        \small{(b) Set size distribution.} \\

        \end{tabular}
        \caption{Analysis on empirical coverage and set size distribution on representative datasets. Results using CLIP ViT-B/16 as backbone, $N=C\times16$ samples for calibration.}
        \label{fig:cov_per_dataset}
    \end{center}
\end{figure*}

\textbf{Comparison against transductive adaptation.} Recently explored unsupervised transductive solvers (ICP-T in Table~\ref{tab:conformal_results}) improve over vanilla ICP in accuracy and set efficiency while maintaining the same calibration sample and therefore similar dispersion in coverage distribution. However, their unsupervised adaptation underperforms SCP in accuracy and therefore yields larger conformal sets.

\textbf{Computational efficiency.} Figure~\ref{fig:latency}(a) showcases the latency scaling, in images per second, for full conformal procedures with the number of categories. Our experiments estimate that running the full conformal procedure using the iterative GD solver on ImageNet would require $\sim500\,\text{s}$ per image, which is infeasible, as generally assumed in the literature. By using our SO-LDA solver, such latency is dramatically reduced to $\sim0.33 \,\text{s}$, and is further decreased to $\sim30 \,\text{ms}$ when pruning unlikely labels via our proposed T-FCP procedure. Regarding GPU usage, we parallelized the number of tested labels in the FCP procedure to $100$, which maintains, as depicted in Figure~\ref{fig:latency}(b), a top-up peak-GPU memory consumption of $\sim16\,\text{Gb}$ on the largest dataset, i.e., ImageNet. However, as shown in  Figure~\ref{fig:latency}(b), such a limit can be decreased without sacrificing much latency.

\begin{figure}[t]
\centering

\begin{minipage}{0.48\textwidth}
  \centering
    \includegraphics[width=0.90\linewidth]{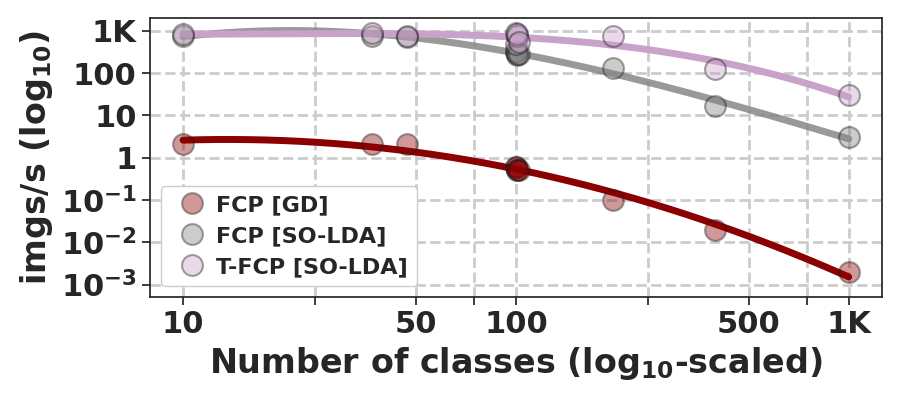} 
\end{minipage}
\hfill
\begin{minipage}{0.48\textwidth}
  \centering
  {\scriptsize
  \begin{tabular}{c r r r}
  
    \toprule
    \multicolumn{1}{c}{\textbf{Dataset}} & \multicolumn{1}{c}{\textbf{C / Iter}} & \multicolumn{1}{c}{\textbf{P-GPU}} & \multicolumn{1}{c}{\textbf{imgs / s}} \\
    \midrule

    \multirow{2}{*}{DTD}
     & 25  & 106.3 MB & 393 \\
     \multirow{2}{*}{($C=37$)} & 50  & 181.0 MB & 745 \\
     & 100 & 181.1 MB & 746 \\
    
     \midrule
    \multirow{2}{*}{ImageNet}
     & 25  &  4.1 GB  & 27.7 \\
     \multirow{2}{*}{($C=1$K)} & 50  &  8.2 GB  & 29.4 \\
     & 100 & 16.2 GB  & 30.3 \\
    
    \bottomrule

  \end{tabular} 
  }
\end{minipage} \\[2ex]
\begin{minipage}{0.48\textwidth}
\centering
\small{(a) Latency scaling with number of classes.}
\end{minipage}
\begin{minipage}{0.48\textwidth}
\centering
\small{(b) Label parallelization (\textbf{C/Iter}) in T-FCP.}
\end{minipage}

\caption{Computational efficiency analysis. Results using ViT-B/16, $N=C\times16$ and $\alpha=0.05$. All our experiments were conducted on a single NVIDIA A100-PCIE-40GB.}
\label{fig:latency}
\end{figure}

\subsection{In-depth studies}
\label{ssec:analysis}

\textbf{Label pruning stage.} Figure~\ref{fig:alpha_icp}~(a) provides a visualization of the pruning capabilities of the inductive stage of our T-FCP procedure across datasets. Figure~\ref{fig:alpha_icp}~(a, \textit{left}) shows that most categories are discarded with small $\alpha_{\text{ICP}}$, specifically the elbow is found near $1\%$ of the allowed error rate. Therefore, most computational efficiency gains can be achieved without sacrificing much of the error-rate budget for the full conformal stage. Figure~\ref{fig:alpha_icp}~(a, \textit{center}) shows a correlation between the zero-shot performance and the pruning capabilities; thus, T-FCP will be more computationally efficient the stronger the zero-shot model initially is. However, upon a closer look to Figure~\ref{fig:alpha_icp}~(a, \textit{center} and \textit{right}), large-scale datasets, such as SUN397 (\tikz[baseline=-0.5ex]\draw[fill=gray, opacity=0.8] (0,0) circle (0.8ex);) or ImageNet (\tikz[baseline=-0.5ex]\draw[fill=pink, opacity=0.8] (0,0) circle (0.8ex);), present larger pruning ratios than their zero-shot accuracy (nearly $60-70\%$) would imply. These results suggest that our T-FCP procedure is especially effective on the largest-scale datasets, where inter-class heterogeneities are greater and therefore categories with unrelated visual content are more easily pruned.

\textbf{Conservativeness and computational feasibility trade-off.} In Figure~\ref{fig:alpha_icp}~(b) we explore the effect of pruning at different $\alpha_{\text{ICP}}$ on latency gains, as well as the set efficiency loss. As theoretically expected, larger $\alpha_{\text{ICP}}$ provide faster runtimes in the FCP stage according to Figure~\ref{fig:alpha_icp}~(b, \textit{left}). In terms of set efficiency, i.e, conservativeness, the produced set distribution is bounded between the ICP and FCP, as shown in Figure~\ref{fig:alpha_icp}~(b, \textit{right}), providing a user with control over the computational feasibility and set efficiency sacrifice via $\alpha_{\text{ICP}}$. If the latter is a priority, then T-FCP provides gains over SCP only when using $\alpha_{\text{ICP}}<2.5\%$. On the other hand, if coverage consistency is the main desideratum, all T-FCP configurations would yield a smaller variance in the finite-sample coverage distribution, similar to ICP, and therefore the decision on how to set $\alpha_{\text{ICP}}$ is constrained only by the available latency margin.

\begin{figure*}[h!]
    \begin{center}
        \setlength{\tabcolsep}{1.8pt}
        \begin{tabular}{ccccc}

         \includegraphics[width=.19\linewidth]{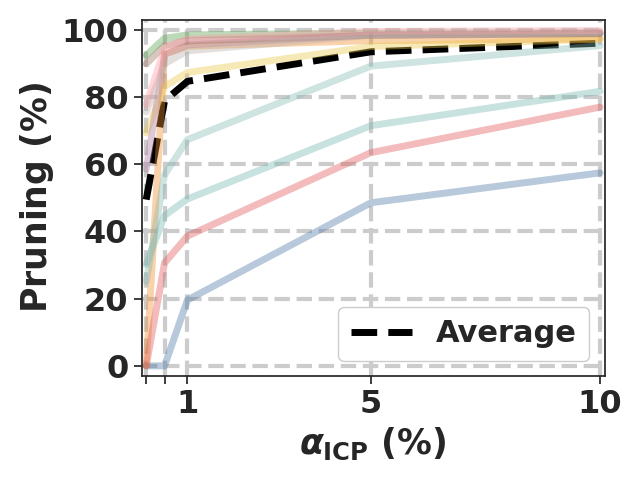}          &
         \includegraphics[width=.19\linewidth]{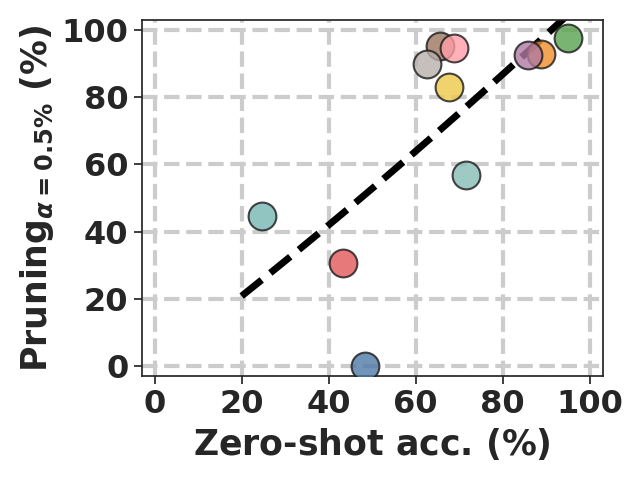}           & 
        \includegraphics[width=.19\linewidth]{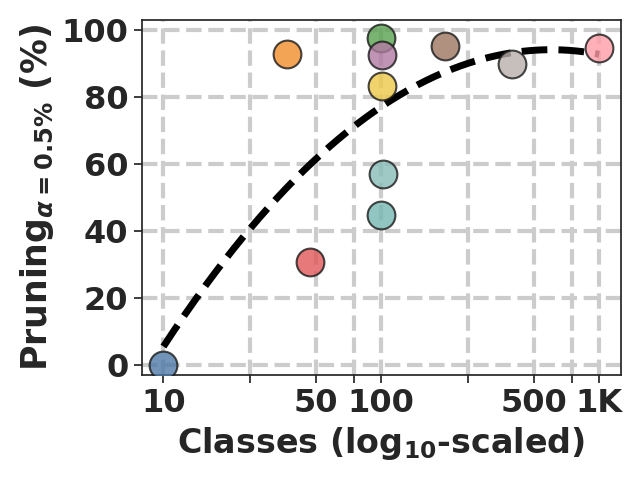}
           & 
        \includegraphics[width=.19\linewidth]{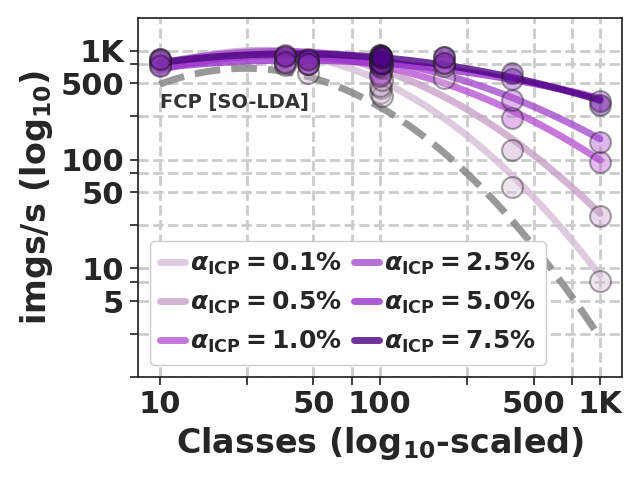}  
           & 
        \includegraphics[width=.19\linewidth]{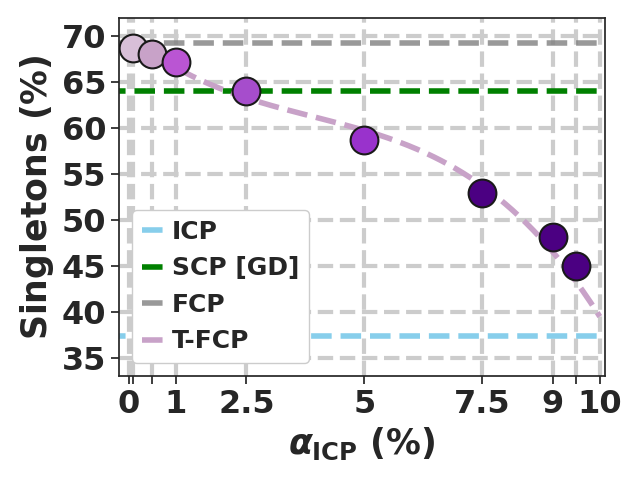}  \\

        \multicolumn{3}{c}{\small{(a) Pruning control at ICP stage (results per dataset).}} & \multicolumn{2}{c}{ \small{(b) Latency vs. convervativeness.}}

        \end{tabular}
        \caption{Ablation studies on T-FCP. Results using CLIP ViT-B/16 as backbone, $N=C\times16$ samples for calibration, and conformal sets created at $\alpha=0.1$ error rate.}
        \label{fig:alpha_icp}
    \end{center}
\end{figure*}

\textbf{Robustness to smaller calibration samples.} Figure~\ref{fig:data_regimes}~(a-d) tests the robustness of the different conformal procedures under smaller calibration sets. Focusing on the empirical coverage dispersion Figure~\ref{fig:data_regimes}~(b), as theoretically expected, all procedures reduce it as more calibration data become available. In this context, T-FCP enables more consistent coverage than SCP across all data regimes, while providing consistently more efficient sets (Figure~\ref{fig:data_regimes}~(c-d)). In the extreme low-data regime, i.e., $N=C\times 4$, we observe a slight instability of the full conformal procedure, Figure~\ref{fig:data_regimes}~(a), which provides slightly lower marginal coverage than desired, an empirical instability already observed in prior work~\cite{martinez2023approximating}. However, we would like to point out that the large variance of the SCP procedure ($\sim\pm4.5$) makes its realized coverage also less stable in this regime.

\textbf{SO-LDA configuration and performance.} Figure~\ref{fig:data_regimes}~(e) compares the discriminative performance of the different relevant linear solvers across increasing data regimes. First, it is worth noting that methods that rely solely on class centers, e.g., Simple-shot (SS)~\cite{simpleshot,fca25}, despite their computational efficiency, fall short in performance compared to the best performing method, i.e., using iterative gradient-descent updates (GD), which makes them uncompetitive for full conformal procedures compared to SCP. Our proposed stabilized LDA solver with zero-shot guided feature centering in Eq.~\ref{eq:so_residual} consistently improves the performance of the learned prototypes compared to vanilla LDA~\cite{wang2024a}. We attribute this behavior to the fact that visual examples within a category may be highly correlated in a specific visual domain, whereas textual information can improve the estimation of the main directions of dispersion in class distributions. Notably, SO-LDA reached performances closer to the expensive iterative GD solver, just $\sim1.5\%$ below. The constraint that learned prototypes remain close to their initial configurations in Eq.~\ref{eq:so_mean} helped stabilize performance gains in lower data regimes, as already observed in prior work on transfer learning from VLMs~\cite{Khattak_2023_ICCV,lin2023multimodality,clap24,sstext25}.

\begin{figure*}[h!]
    \begin{center}
        \setlength{\tabcolsep}{1.9pt}
        \begin{tabular}{ccccc}

         \includegraphics[width=.19\linewidth]{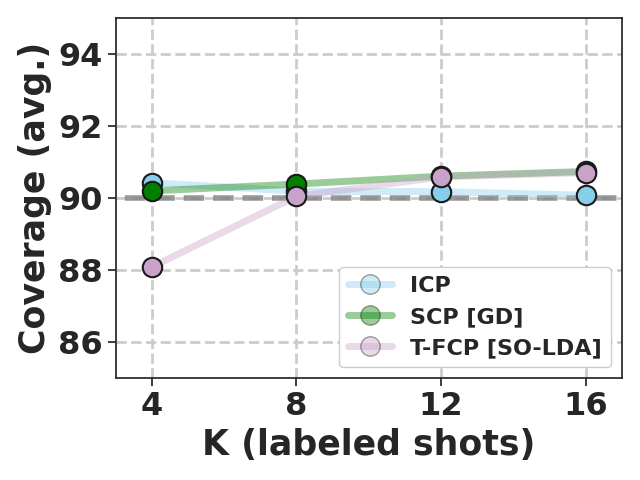}          &
          \includegraphics[width=.19\linewidth]{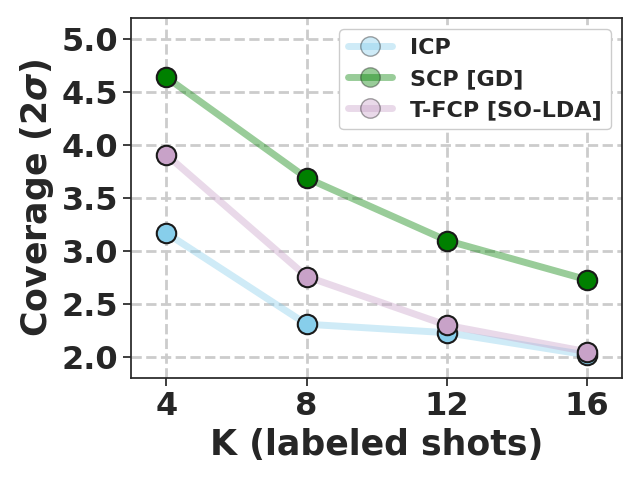}          &
         \includegraphics[width=.19\linewidth]{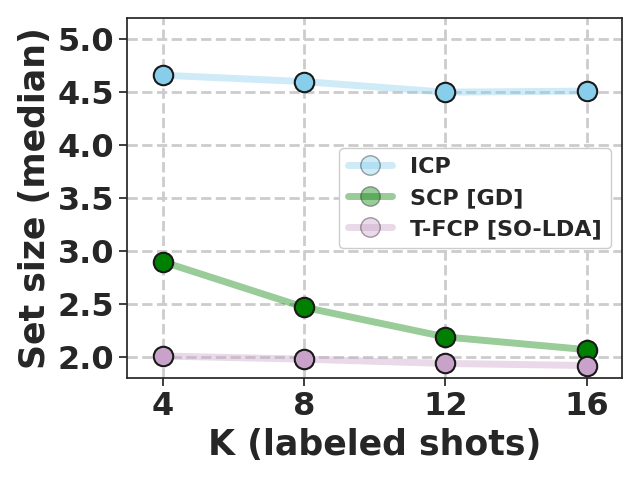}           & 
        \includegraphics[width=.19\linewidth]{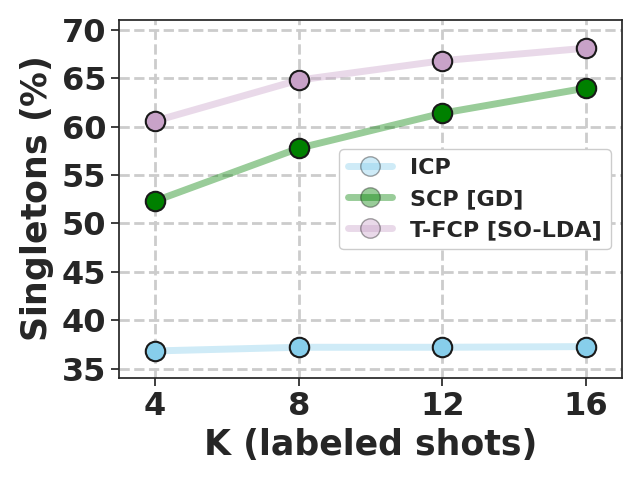}           & 
         \includegraphics[width=.19\linewidth]{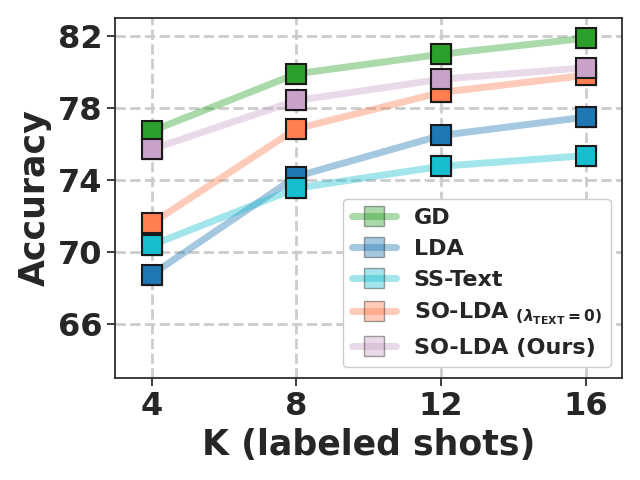}           \\
         \small{(a)} & \small{(b)} & \small{(c)} & \small{(d)} & \small{(e)}

        \end{tabular}
        \caption{Studies at lower data regimes: (a-d) conformal procedures, and (e) few-shot solvers. Results using CLIP ViT-B/16 as backbone, and conformal sets at $\alpha=0.1$, averaged for 11 datasets.}
        \label{fig:data_regimes}
    \end{center}
\end{figure*}

\textbf{Performance atop other VLMs.} We evaluate the conformal procedures on additional backbones, including CLIP~\cite{radford2021learning} and MetaCLIP~\cite{xu2024demystifying}, in Figure~\ref{fig:backbones}. The merits of T-FCP hold across all backbones, providing consistently more concentrated empirical coverage distributions than SCP, with variances $\sim30\%$ smaller, while maintaining valid coverage. Moreover, predictive accuracy and set efficiency are consistently improved, especially compared to ICP. This is even the case for larger-scale backbones, e.g., MetaCLIP ViT-L/14, for which the median set size is reduced by $31\%$, without sacrificing finite-sample coverage stability. Despite the strong reliance of T-FCP and our SO-LDA solver on zero-shot prototypes, performance remains robust also in smaller backbones, e.g., ViT-B/32.

\begin{figure*}[h!]
    \begin{center}
        \setlength{\tabcolsep}{1.9pt}
        \begin{tabular}{cc}

        \includegraphics[width=.49\linewidth]{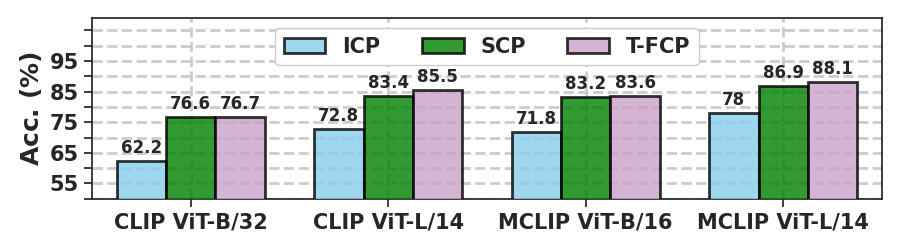} &
        \includegraphics[width=.49\linewidth]{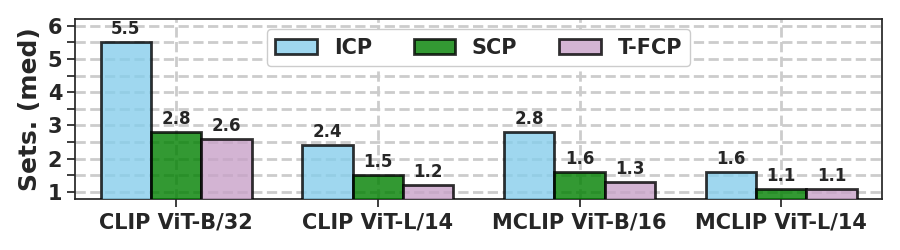} \\
        \includegraphics[width=.49\linewidth]{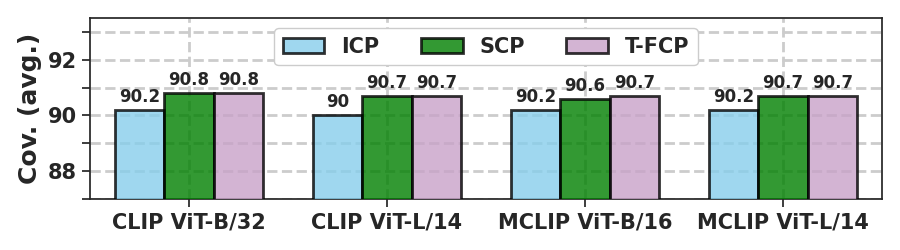} & 
        \includegraphics[width=.49\linewidth]{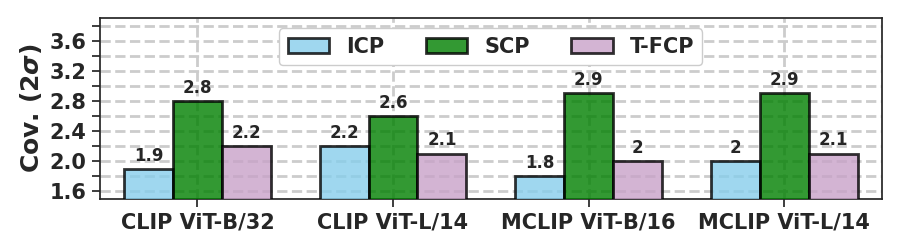} \\

        \end{tabular}
        \caption{Generalization atop other CLIP~\cite{radford2021learning} and MetaCLIP~\cite{xu2024demystifying} backbones of conformal procedures. Results using $N=C\times16$ samples and $\alpha=0.1$, averaged across 11 datasets.}
        \label{fig:backbones}
    \end{center}
\end{figure*}
\vspace{-5mm}

% ---- Conclusions ----
\section{Discussion}
\label{sec:conc}

We presented T-FCP and SO-LDA, enabling full conformal prediction on large-scale image datasets by leveraging zero-shot VLMs. Our experiments highlight that data-efficient procedures yield more stable empirical coverage than split-based alternatives and underscore the importance of data efficiency in conformal image classification.

\textbf{Limitations.} The proposed linear solver is still less performant than iterative gradient-based methods. Although this gap is partially compensated for in FCP by fitting on each candidate-augmented dataset, we cannot guarantee that the resulting prediction sets will always be more efficient than those from SCP, which can use more expensive solvers during its adaptation stage. A further limitation is that our main SO-LDA implementation uses diagonal loading to stabilize the precision matrix, thereby approximating recomputing a fully regularized LDA estimator for each augmented dataset. Therefore, although T-FCP offers theoretically valid coverage, the same is not true when it is combined with our SO-LDA solver. In practice, ablations in Appendix~\ref{app_solda} show negligible differences against exact implementations. Finally, the coverage guarantee of T-FCP is based on an additive union-bound allocation across the pruning and full-conformal stages. Consequently, the resulting sets may be slightly more conservative than FCP's. Hence, FCP may be preferred for smaller-scale tasks, e.g., those involving tens or a few hundred classes, especially when coupled with our online LDA solver.

% ---- References ----
\bibliographystyle{plain} % or your preferred style (see below)
\bibliography{refs.bib} % your bib file name without the .bib extension

% ---- Appendix ----
\newpage
\appendix

\section{Broader Impacts}
\label{app_impact}

\textbf{Broader impacts.} Conformal prediction is often motivated by providing theoretical guarantees in high-stakes deployment. However, it should be interpreted with care. Marginal coverage does not imply conditional or subgroup coverage, and guarantees need not hold under distribution shift. Moreover, finite-sample empirical coverage can fall below the nominal level in a real-world deployment, especially with data-inefficient pipelines. In this work, we directly address the latter use case, enabling more stable conformal image classifiers via full conformal prediction. We do not anticipate negative societal impacts from our work beyond those of standard conformal prediction, and we identify no domain-specific risks requiring special attention.

\section{Proofs}
\label{app_proof}

\begin{proof}[Proof of Proposition~\ref{prop:intersection}]
According to Eq \ref{combination}:
\begin{align}
    \mathbb{P}\big(Y \notin \mathcal{C}_{1\cap2}(\mathbf{x})\big)
    &= \mathbb{P}\big( Y \notin \mathcal{C}_1(\mathbf{x}) \;\cup\; Y \notin \mathcal{C}_2(\mathbf{x}) \big).
\end{align}
By the union bound,
\begin{align}
    \mathbb{P}\big(Y \notin \mathcal{C}_{1\cap2}(\mathbf{x})\big)
    \leq \mathbb{P}\big(Y \notin \mathcal{C}_1(\mathbf{x})\big)
    + \mathbb{P}\big(Y \notin \mathcal{C}_2(\mathbf{x})\big).
\end{align}
Using the marginal coverage guarantees for each individual procedure in Eq. \ref{alpha1alpha2},
\begin{align}
    \mathbb{P}\big(Y \notin \mathcal{C}_{1\cap2}(\mathbf{x})\big) \leq \alpha_1 + \alpha_2,
\end{align}
which implies the result.
\end{proof}

\section{Stability of SO-LDA}
\label{app_solda}
In this section, we further analyze SO-LDA, particularly the implications of the mild asymmetry introduced by the diagonal covariance matrix regularization used in the main implementation, whose diagonal is estimated only from calibration data to enable efficient rank-one updates. We compare such a solution with different configurations:
\begin{itemize}
    \item \textbf{Ridge}: the regularization is imposed via a fixed identity loading, $\mathbf{S}^{\mathrm{reg}}_{(N)}=\mathbf{S}_{(N)}+\lambda_{\mathrm{RIDGE}}\mathbf{I}$, which is compatible with permutation-invariant rank-one updates under fixed regularization.
    \item \textbf{Non-online}: for each candidate sample-label pair, it recomputes the class means and the covariance matrix from the full augmented dataset, then directly inverts the regularized covariance matrix. This solver increases the per-label computational cost to $\mathcal{O}(NF^2 + F^3)$.
    \item \textbf{Unstable}: does not use zero-shot prototypes for the residuals estimation as in Eq.~\ref{eq:so_residual}. Hence, the centered features are calculated asymmetrically between calibration and test data.
\end{itemize}
Comparative performance on different data regimes is presented in Table~\ref{tab:solda}. The results suggest that the dominant source of instability is the asymmetric centering of the residuals used for covariance estimation, with the \textbf{Unstable} configuration yielding the largest marginal coverage gap and lower ratio of valid occurrences. The fixed-\textbf{Ridge} configuration is the cleanest variant with respect to symmetric online updates. However, its discriminative performance is lower than that of our diagonal regularization in SO-LDA. Moreover, both exhibit a similar coverage distribution, suggesting that the asymmetry introduced by our implementation choice is negligible. Also, the performance levels achieved by SO-LDA are similar to those obtained with more expensive, exact \textbf{Non-online} updates. 

The under-coverage for $K=4$ is observed in both online and exact recomputation variants, suggesting it is not primarily caused by the diagonal-loading approximation. As this instability in exact FCP also appears in prior work~\cite{martinez2023approximating} (e.g., see their results at their Figure 4(d)), we leave a complete characterization of this regime to future work.

\begin{table}[h!]
\setlength{\tabcolsep}{2.0pt}  % default is ~6pt
\centering
\caption{Study of the SO-LDA solver configurations under T-FCP. Average results for 10 datasets. In contrast to the main manuscript, these studies use $20$ random seeds and exclude ImageNet.}
\label{tab:solda}
\small

\begin{tabular}{cc}

\begin{tabular}{lll c ccc cc}
\toprule
\multicolumn{3}{c}{}
& \multicolumn{3}{c}{Cov. ($\alpha = 0.10$)} 
& \multicolumn{2}{c}{Set size} \\
\cmidrule(lr){4-6} \cmidrule(lr){7-8}
& & Acc. 
& Avg. & $2\sigma$ & \%Valid 
& Mean & \%Sing. \\ \midrule
\textbf{SO-LDA}                 & & 81.0 & 90.8 & 2.1 & 80.0 & 2.3 & 70.3   \\ \midrule
Non-online                      & & 81.0 & 90.8 & 2.1 & 80.0 & 2.3 & 70.2  \\
Ridge                           & & 78.8 & 90.7 & 2.3 & 76.0 & 2.4 & 64.6  \\ 
Unstable                        & & 78.8 & 90.4 & 2.3 & 67.5 & 2.6 & 62.3  \\
\bottomrule
\end{tabular} 

& 

\begin{tabular}{lll c ccc cc}
\toprule
\multicolumn{3}{c}{}
& \multicolumn{3}{c}{Cov. ($\alpha = 0.10$)} 
& \multicolumn{2}{c}{Set size} \\
\cmidrule(lr){4-6} \cmidrule(lr){7-8}
& & Acc. 
& Avg. & $2\sigma$ & \%Valid 
& Mean & \%Sing. \\ \midrule
\textbf{SO-LDA}                & & 76.4 & 88.2 & 4.2 & 27.0 & 3.1 & 61.6   \\ \midrule
Non-online                     & & 76.4 & 88.3 & 4.2 & 27.5 & 3.2 & 61.3  \\
Ridge                          & & 74.1 & 88.0 & 4.3 & 26.0 & 2.9 & 58.4  \\
Unstable                       & & 70.1 & 86.0 & 5.3 & 13.0 & 4.6 & 45.9  \\
\bottomrule
\end{tabular} \\
(a) $N=C \times 16$ & (b) $N=C \times 4$ \\

\end{tabular}

\end{table}

\section{Datasets}
\label{app_datasets}

A summary of the datasets, number of classes, tasks, and number of samples in the used test partition is depicted in Table~\ref{datasets}.

\begin{table*}[!ht]
\centering
\caption{Datasets overview. “$b$"/“$i$" denotes balanced/imbalanced.}
\label{datasets}
\small
\vspace{1mm}
\begin{tabular}{lrrrrcl}
\toprule
\multicolumn{1}{c}{Dataset} & Classes & \multicolumn{3}{c}{Splits} & $b$/$i$ & \multicolumn{1}{c}{Task description} \\
                            & & \multicolumn{1}{c}{Train} & \multicolumn{1}{c}{Val} & \multicolumn{1}{c}{Test} & & \\
\midrule
EuroSAT \cite{eurosat}                  & 10    & 13,500 & 5,400  & 8,100   & $i$ & Satellite image classification.  \\  
OxfordPets \cite{oxfordpets}            & 37    & 2,944  & 736    & 3,669   & $i$ & Pets classification.                     \\
DTD \cite{dtd}                          & 47    & 2,820  & 1,128  & 1,692   & $b$ & Textures classification.                 \\
FGVCAircraft \cite{aircraft}            & 100   & 3,334  & 3,333  & 3,333   & $i$ & Aircraft classification. \\    
Caltech101 \cite{caltech}               & 100   & 4,128  & 1,649  & 2,465   & $i$ & Natural objects classification.          \\
UCF101 \cite{ucf101}                    & 101   & 7,639  & 1,898  & 3,783   & $i$ & Action recognition.    \\           
Food101 \cite{food101}                  & 101   & 50,500 & 20,200 & 30,300  & $b$ & Foods classification.                    \\
Flowers102 \cite{flowers102}            & 102   & 4,093  & 1,633  & 2,463   & $u$ & Flowers classification.                  \\
StanfordCars \cite{stanfordcars}        & 196   & 6,509  & 1,635  & 8,041   & $i$ & Cars classification.                     \\
SUN397 \cite{sun397}                    & 397   & 15,880 & 3,970  & 19,850  & $b$ & Scenes classification.  \\ 
ImageNet \cite{deng2009imagenet}        & 1,000 & 1.28M & - & 50,000 & $b$  & Natural objects recognition. \\
                                 
\bottomrule
\end{tabular}
\end{table*}

\section{Baselines}
\label{app_baselines}

\textbf{Gradient-descent linear probe (GD).} We followed training details similar to~\cite{clap24}, which remain fixed for all datasets. The learned class weights are initialized with zero-shot prototypes. Full-batch gradient descent is performed for $300$ iterations, minimizing cross-entropy. SGD is used as the optimizer with momentum $0.9$. A cosine-annealing scheduler with a base learning rate of $0.1$ is applied to ensure a small step size and proper convergence.

\textbf{ICP-T via Conf-OT~\cite{confot25}.} The logits are extracted using the zero-shot prototypes, and label-marginal distributions are estimated from calibration data. The Optimal Transport solver is applied to the joint calibration and test sets, with the entropic weight $\tau=1$ and $T_{\text{OT}} = 3$, as default values in~\cite{confot25}.

\textbf{ICP-T via TIM~\cite{scat}.} We use the Transductive Information Maximization solver with regularization on the expected label-marginal distribution, estimated from calibration labels. Training is carried out on the joint calibration and test samples. The relative weight $\lambda$ is set to $1.0$, as in~\cite{scat}. Similarly to GD, full-batch gradient descent is performed for $300$ iterations, in this case optimizing the mutual information criteria. Other details are as in GD, since they provided satisfactory convergence.

\section{Detailed Results}

\subsection{Additional configurations}

We report the performance of the proposed solver SO-LDA for split conformal prediction in Table~\ref{tab:conformal_results_ext}.
\begin{table}[h!]
\setlength{\tabcolsep}{2.6pt}  % default is ~6pt
\centering
\caption{Comparison of conformal procedures atop CLIP ViT-B/16 using $N=C\times16$. This table reports the results using configurations not included in Table~\ref{tab:conformal_results}.}
\vspace{-1mm}
\label{tab:conformal_results_ext}
\small
\begin{tabular}{lll | c | ccc ccc | ccc ccc}
\toprule
\multicolumn{4}{c}{}
& \multicolumn{3}{c}{Cov. ($\alpha = 0.10$)} 
& \multicolumn{3}{c}{Set size}
& \multicolumn{3}{c}{Cov. ($\alpha = 0.05$)} 
& \multicolumn{3}{c}{Set size} \\
\cmidrule(lr){5-7} \cmidrule(lr){8-10}
\cmidrule(lr){11-13} \cmidrule(lr){14-16}
& & & Acc. 
& Avg. & $2\sigma$ & \%Valid 
& Mean & Med. & \%Sing. 
& Avg. & $2\sigma$ & \%Valid 
& Mean & Med. & \%Sing. \\
\midrule
SCP & SO-LDA & & 79.5 & 90.8 & 3.5 & 71.3 & 2.7 & 2.3 & 69.4 & 95.6 & 2.4 & 77.0 & 4.1 & 3.3 & 58.4 \\
\bottomrule
\end{tabular}
\end{table}

\subsection{Table~\ref{tab:conformal_results}: detailed results per dataset}
\label{app_detailedresults}

The detailed results of the conformal procedures' performance across datasets are reported in Table~\ref{tab:conformal_results_detailed}.

\begin{table}[h!]
\setlength{\tabcolsep}{2.2pt}  % default is ~6pt
\centering
\caption{Per-dataset conformal prediction results of average performances of the main conformal procedures reported in Table~\ref{tab:conformal_results}, using CLIP ViT-B/16.}
\vspace{-1mm}
\label{tab:conformal_results_detailed}
\small
\begin{tabular}{llll | c | ccc ccc | ccc ccc}
\toprule
& \multicolumn{4}{c}{}
& \multicolumn{3}{c}{Cov. ($\alpha = 0.10$)} 
& \multicolumn{3}{c}{Set size}
& \multicolumn{3}{c}{Cov. ($\alpha = 0.05$)} 
& \multicolumn{3}{c}{Set size} \\
\cmidrule(lr){6-8} \cmidrule(lr){9-11}
\cmidrule(lr){12-14} \cmidrule(lr){15-17}
& & & & Acc. 
& Avg. & $2\sigma$ & \%Valid 
& Mean & Med. & \%Sing. 
& Avg. & $2\sigma$ & \%Valid 
& Mean & Med. & \%Sing. \\
\midrule
\multirow{3}{*}{\scriptsize{\rotatebox{90}{EuroSAT}}} &
ICP & ZS & & 48.4 & 90.4 & 6.1 & 62.0 & 4.2 & 4.4 & 9.4 & 95.6 & 4.8 & 66.0 & 5.2 & 5.3 & 6.6 \\
& SCP & GD & & 83.4 & 89.8 & 6.5 & 50.0 & 1.4 & 1.1 & 73.4 & 96.1 & 4.7 & 76.0 & 2.2 & 1.8 & 43.8 \\
& \textbf{T-FCP} & 
\textbf{SO-LDA} & & 83.0 & 90.2 & 5.3 & 58.0 & 1.4 & 1.0 & 71.4 & 95.9 & 3.9 & 90.0 & 2.0 & 1.9 & 42.0 \\
\midrule
\multirow{3}{*}{\scriptsize{\rotatebox{90}{Pets}}} & 
ICP & ZS & & 88.8 & 89.9 & 3.2 & 54.0 & 1.0 & 1.0 & 92.9 & 95.1 & 2.0 & 68.0 & 1.2 & 1.0 & 80.4 \\
& SCP & GD & & 92.3 & 90.0 & 4.9 & 65.0 & 1.0 & 1.0 & 94.8 & 95.2 & 3.3 & 66.0 & 1.1 & 1.0 & 91.0 \\
& \textbf{T-FCP} & \
\textbf{SO-LDA} & & 92.8 & 90.3 & 2.7 & 70.0 & 1.0 & 1.0 & 93.6 & 95.6 & 2.2 & 82.0 & 1.1 & 1.0 & 91.0 \\
\midrule
\multirow{3}{*}{\scriptsize{\rotatebox{90}{DTD}}} &
ICP & ZS & & 43.4 & 89.9 & 2.9 & 66.0 & 10.8 & 11.1 & 6.75 & 95.0 & 1.9 & 72.0 & 17.2 & 18.6 & 4.3 \\
& SCP & GD & & 66.6 & 90.2 & 4.4 & 55.0 & 3.4 & 2.8 & 28.3 & 95.1 & 2.6 & 72.0 & 6.2 & 5.3 & 16.7 \\
& \textbf{T-FCP} & 
\textbf{SO-LDA} & & 69.7 & 90.3 & 2.9 & 66.0 & 3.2 & 2.0 & 43.2 & 95.0 & 2.0 & 70.0 & 6.1 & 3.4 & 29.7 \\
\midrule
\multirow{3}{*}{\scriptsize{\rotatebox{90}{Aircraft}}} & 
ICP & ZS & & 24.8 & 90.0 & 1.3 & 76.0 & 18.4 & 18.1 & 0.8 & 95.2 & 1.2 & 90.0 & 28.6 & 28.0 & 0.1 \\
& SCP & GD & & 41.4 & 90.1 & 2.3 & 65.0 & 8.8 & 8.8 & 3.7 & 95.0 & 1.8 & 70.0 & 12.7 & 12.7 & 2.1 \\
& \textbf{T-FCP}  & 
\textbf{SO-LDA} & & 39.6 & 90.2 & 2.1 & 70.0 & 9.4 & 8.9 & 7.2 & 95.1 & 1.5 & 78.0 & 13.1 & 12.5 & 5.9 \\
\midrule
\multirow{3}{*}{\scriptsize{\rotatebox{90}{Caltech}}} & 
ICP & ZS  & & 95.1 & 90.4 & 1.6 & 88.0 & 0.9 & 1.0 & 93.4 & 96.2 & 1.0 & 100.0 & 1.0 & 1.0 & 94.0 \\
& SCP & GD & & 97.8 & 94.9 & 1.4 & 100.0 & 1.0 & 1.0 & 95.5 & 97.8 & 0.9 & 100.0 & 1.0 & 1.0 & 99.0 \\
& \textbf{T-FCP} & 
\textbf{SO-LDA} & & 97.9 & 94.1 & 0.9 & 100.0 & 1.0 & 1.0 & 94.7 & 97.3 & 0.9 & 100.0 & 1.0 & 1.0 & 98.4 \\
\midrule
\multirow{3}{*}{\scriptsize{\rotatebox{90}{UCF101}}} & 
ICP & ZS & & 67.6 & 90.1 & 1.7 & 74.0 & 2.9 & 2.0 & 29.3 & 95.1 & 1.4 & 80.0 & 5.1 & 3.9 & 16.4 \\
& SCP & GD & & 89.3 & 90.3 & 3.0 & 70.0 & 1.0 & 1.0 & 90.5 & 95.1 & 1.8 & 68.0 & 1.3 & 1.0 & 75.1 \\
& \textbf{T-FCP} & 
\textbf{SO-LDA} & & 89.6 & 90.5 & 1.9 & 86.0 & 1.0 & 1.0 & 94.8 & 95.4 & 1.4 & 86.0 & 1.3 & 1.0 & 79.2 \\
\midrule
\multirow{3}{*}{\scriptsize{\rotatebox{90}{Food101}}} & 
ICP & ZS & & 86.0 & 90.1 & 2.0 & 72.0 & 1.1 & 1.0 & 85.1 & 95.1 & 1.4 & 72.0 & 1.6 & 1.0 & 63.4 \\
& SCP & GD & & 86.3 & 90.0 & 2.9 & 60.0 & 1.1 & 1.0 & 86.8 & 94.8 & 2.0 & 60.0 & 1.5 & 1.0 & 66.5 \\
& \textbf{T-FCP} & 
\textbf{SO-LDA} & & 86.4 & 90.4 & 1.9 & 80.0 & 1.2 & 1.0 & 87.8 & 95.4 & 1.6 & 90.0 & 1.7 & 1.0 & 68.3 \\
\midrule
\multirow{3}{*}{\scriptsize{\rotatebox{90}{Flowers}}} & 
ICP & ZS & & 71.5 & 90.3 & 0.8 & 98.0 & 4.9 & 3.9 & 19.7 & 95.2 & 0.8 & 96.0 & 11.1 & 9.4 & 5.67 \\
& SCP & GD & & 96.5 & 92.2 & 2.1 & 100.0 & 0.9 & 1.0 & 93.0 & 96.3 & 1.4 & 98.0 & 1.0 & 1.0 & 96.9 \\
& \textbf{T-FCP} & 
\textbf{SO-LDA} & & 95.7 & 90.9 & 1.6 & 96.0 & 0.9 & 1.0 & 92.9 & 95.8 & 1.1 & 100.0 & 1.0 & 1.0 & 96.1 \\
\midrule
\multirow{3}{*}{\scriptsize{\rotatebox{90}{Cars}}} & 
ICP & ZS & & 65.5 & 90.1 & 1.0 & 86.0 & 2.4 & 2.0 & 27.6 & 95.2 & 0.9 & 88.0 & 3.4 & 3.0 & 17.3 \\
& SCP & GD & & 78.6 & 90.3 & 1.8 & 90.0 & 1.5 & 1.0 & 59.3 & 95.2 & 1.1 & 86.0 & 2.0 & 2.0 & 38.3 \\
& \textbf{T-FCP} & 
\textbf{SO-LDA} & & 80.8 & 90.4 & 1.2 & 90.0 & 1.4 & 1.0 & 65.4 & 95.3 & 0.9 & 96.0 & 1.9 & 2.0 & 45.9 \\
\midrule
\multirow{3}{*}{\scriptsize{\rotatebox{90}{SUN397}}} & 
ICP & ZS & & 62.6 & 90.2 & 1.0 & 90.0 & 3.5 & 3.0 & 16.2 & 95.1 & 0.7 & 98.0 & 6.6 & 5.3 & 7.1 \\
& SCP & GD & & 73.7 & 90.0 & 1.5 & 80.0 & 1.9 & 2.0 & 40.8 & 95.1 & 1.1 & 84.0 & 3.2 & 2.9 & 22.0 \\
& \textbf{T-FCP} & 
\textbf{SO-LDA} & & 74.3 & 90.3 & 1.4 & 84.0 & 2.0 & 1.2 & 51.3 & 95.3 & 1.0 & 94.0 & 3.4 & 2.0 & 31.2 \\
\midrule
\multirow{3}{*}{\scriptsize{\rotatebox{90}{ImageNet}}} & 
ICP & ZS & & 68.7 & 90.0 & 0.7 & 90.0 & 2.8 & 2.0 & 28.9 & 95.0 & 0.6 & 94.0 & 5.5 & 4.0 & 15.0 \\
& SCP & GD & & 72.5 & 90.0 & 0.8 & 85.0 & 2.3 & 2.0 & 36.9 & 95.0 & 0.5 & 94.0 & 4.2 & 3.0 & 20.1 \\
& \textbf{T-FCP} & 
\textbf{SO-LDA} & & 72.9 & 90.3 & 0.7 & 100.0 & 2.5 & 2.0 & 46.7 & 95.4 & 0.4 & 100.0 & 5.7 & 3.0 & 27.9 \\
\bottomrule
\end{tabular}
\end{table}

\subsection{Results using APS}
\label{results_aps}

We report results using APS~\cite{aps} as the non-conformity score in Table~\ref{tab:results_aps}. They demonstrate that the gains in empirical coverage stability are not specific to LAC, but also generalize to APS: T-FCP reports smaller coverage dispersion and a larger ratio of realizations with satisfied coverage than ICP and SCP. We also incorporate class-conditional average coverage gap (CCV), the reference metric used for measuring adaptiveness \cite{ding2024class}, where T-FCP remains consistent.

\begin{table}[h!]
\setlength{\tabcolsep}{2.6pt}  % default is ~6pt
\centering
\caption{Results using APS. SO-LDA is used as solver, with $N=C\times16$. Average results for 10 datasets. These studies use 20 random seeds and exclude ImageNet.}
\vspace{-1mm}
\label{tab:results_aps}
\small
\begin{tabular}{ll | cccc ccc | cccc ccc}
\toprule
\multicolumn{2}{c}{}
& \multicolumn{4}{c}{Cov. ($\alpha = 0.10$)} 
& \multicolumn{3}{c}{Set size}
& \multicolumn{4}{c}{Cov. ($\alpha = 0.05$)} 
& \multicolumn{3}{c}{Set size} \\
\cmidrule(lr){3-6} \cmidrule(lr){7-9}
\cmidrule(lr){10-13} \cmidrule(lr){14-16}
&  
& Avg. & $2\sigma$ & \%Valid & CCV
& Mean & Med. & \%Sing. 
& Avg. & $2\sigma$ & \%Valid & CCV 
& Mean & Med. & \%Sing. \\
\midrule
ICP & & 90.4 & 2.2 & 81.5 & 8.2 & 6.3 & 4.9 & 34.4 & 95.3 & 1.5 & 85.0 & 5.0 & 10.1 & 8.0 & 27.1 \\
SCP & & 90.1 & 3.1 & 67.5 & 6.3 & 3.0 & 2.2 & 54.0 & 95.1 & 2.4 & 70.0 & 5.0 & 4.3 & 3.1 & 46.9 \\
T-FCP & & 90.4 & 1.7 & 84.0 & 6.4 & 3.3 & 2.3 & 52.6 & 95.3 & 1.4 & 89.0 & 4.9 & 4.5 & 2.9 & 46.9 \\
\bottomrule
\end{tabular}
\end{table}

It is worth noting that adaptive scores such as APS/RAPS require ranking or sorting over the label space. This adds non-negligible cost inside the candidate-wise FCP loop, which conflicts with our scalability focus. For example, for ImageNet, T-FCP [$\alpha_{\text{ICP}}=0.5\%, \ \alpha_{\text{FCP}}=9.5\%$] using LAC has a latency of nearly 30 ms/image, and it increases to nearly 220 ms/image using APS, which makes it currently less appealing to use for the largest-scale datasets for T-FCP and unfeasible for the more expensive vanilla FCP. 

\subsection{Extension to unimodal models}
\label{/results_unimodal}

The T-FCP construction is model-agnostic in principle: any valid low-cost conformal predictor could be used as the pruning stage, and any valid or efficient FCP scoring rule could be used as the second stage. Beyond VLMs, T-FCP can therefore be explored for unimodal models as long as an initial pruning proxy exists, e.g., a pretrained classifier over the target label space. 

To test this, we next transfer a standard ResNet-50 model pretrained on ImageNet to ImageNet-V2 (accuracy 69.1\%), since the two datasets share the same label space. T-FCP [$\alpha_{\text{ICP}}=0.5\%, \ \alpha_{\text{FCP}}=4.5\%$] early prunes 65\% of the labels, and provides a more stable empirical coverage distribution than split-conformal alternatives. Specific results are presented below, in Table~\ref{res:unimodal}.

\begin{table}[h!]
\setlength{\tabcolsep}{6.0pt}  % default is ~6pt
\centering
\caption{Conformal prediction results on ImageNet-V2 using ResNet-50 pre-trained on ImageNet, LAC, and 8 labeled samples per class as calibration data. Results across 20 seeds.}
\vspace{-1mm}
\label{res:unimodal}
\small
\begin{tabular}{lll | ccc c |}
\toprule
\multicolumn{3}{c}{}
& \multicolumn{3}{c}{Cov. ($\alpha = 0.05$)} 
& \multicolumn{1}{c}{Set size} \\ \cmidrule(lr){3-6} \cmidrule(lr){7-7}
& & & Avg. & $2\sigma$ & \%Valid 
& Mean \\
\midrule
ICP   & & & 95.2 & 1.3 & 80.0 & 12.3 \\
SCP   & & & 95.0 & 1.2 & 70.0 & 11.7 \\
T-FCP & & & 95.1 & 1.2 & 80.0 & 10.6 \\
\bottomrule
\end{tabular}
\end{table}

\end{document}